\documentclass[]{article}
\usepackage{proceed2e}

\usepackage{times}

\usepackage[numbers]{natbib}
\usepackage{helvet}
\usepackage{courier}
\usepackage{url}
\usepackage{amsmath}
\usepackage{amssymb}
\usepackage{amsthm}
\usepackage[vlined, ruled]{algorithm2e}
\usepackage{graphicx}
\usepackage{subcaption}
\usepackage{float}
\usepackage{booktabs}
\usepackage{colortbl}
\usepackage{enumitem}
\usepackage[flushleft]{threeparttable}
\usepackage{gensymb}

\setlist{nolistsep}

\newtheorem{theorem}{Theorem}
\newtheorem{corollary}{Corollary}

\DeclareMathOperator*{\argmax}{\arg\!\max}

\let\emptyset\varnothing

\title{Multi-Object Tracking and Identification over Sets}

\author{Aijun Bai\\ UC Berkeley} % LEAVE BLANK FOR ORIGINAL SUBMISSION.
\begin{document}
\maketitle

\begin{abstract}
\begin{quote}
The ability for an autonomous agent or robot to track and identify potentially multiple objects in a dynamic environment is essential for many applications, such as automated surveillance, traffic monitoring, human-robot interaction, etc. The main challenge is due to the noisy and incomplete perception including
inevitable false negative and false positive errors from a low-level detector. In this paper, we propose a novel multi-object tracking and identification over sets approach to address this challenge. We define joint states and observations both as finite sets, and develop motion and observation functions accordingly. The object identification problem is then formulated and solved by using expectation-maximization methods. The set formulation enables us to avoid directly performing observation-to-object association. We empirically confirm that the overall algorithm outperforms the state-of-the-art in a popular PETS dataset.
\end{quote}
\end{abstract}

\section{Introduction}
The ability to detect/recognize, track and identify multi-objects is essential in domains such as automated surveillance, traffic monitoring, human-robot interaction, etc. Provided with a low-level detector, the main challenge for multi-objects tracking and identification is to sequentially reason about the number of objects, and estimate the state of each object from ambiguous observations, in presence of noisy and incomplete perception including inevitable false and missing detections | false positives and false negatives respectively \cite{martin2010multiple}. Most multi-object tracking (MOT) approaches follow a \emph{tracking-by-detection} paradigm \cite{yilmaz2006object}, where an object detector runs on each frame to recognize all potential objects, and proposes a set of detections as input for a tracker, which estimates the true world state accordingly. Tracking-by-detection algorithms can be roughly classified into two groups: \emph{online} and \emph{offline}. Online tracking intends to recursively estimate the current state given past observations in a filtering way; offline tracking finds an optimal trajectory given the whole sequence of observations. In this paper, we mainly focus on online tracking problems.

In the context of online MOT, most existing approaches assume one or more hypotheses on observation-to-object data-associations, and perform Bayesian filtering on each object separately \cite{yilmaz2006object}. The global nearest neighbor (GNN) based filters find the best hypothesis that minimizes a cost function defined based on total distance or likelihood \cite{blackrnan1999design}. The joint probabilistic data-association (JPDA) based filters update each object by using all detections available weighted according to posterior association probabilities \cite{fortmann1983sonar}. The multiple hypothesis tracking (MHT) method attempts to maintain a set of hypotheses with high posterior probabilities in a tree structure \cite{reid1979algorithm}. Provided that these methods perform separate Bayesian updates by assuming specific data-associations, it is difficult for these methods to recover from wrong assumptions. Instead, we propose to avoid directly performing observation-to-object association, by using a joint state represented as a set to encode the number of objects, and the entire world state in terms of all objects. The filtering step then reasons about the joint state, as well as the data-associations in a Bayesian-optimal way.

The main contribution of this paper is the overall multi-object tracking and identification over sets (MOTIS) algorithm, together with associated techniques we introduce to make it possible, including: 1) the assignment and false-missing pruning strategies to approximate the observation function, 2) a data-association based particle refinement method, 3) a Bayesian density estimation approach to estimate motion and proposal weights, and 4) an expectation-maximization (EM) based object identification procedure to identify each individual object from particles. To compare with existing work, we evaluate MOTIS in standard PETS2009 benchmark data. The experimental results show that our approach outperforms the state-of-the-art in terms of overall tracking accuracy and ID-switch error.

\section{Related Work}
\label{sec:related-work}
Joint multi-object probability density (JMPD) \cite{kreucher2005multitarget} also applies a joint state formalization similar to our work. In each frame, they assume discretized pixel measurements (larger than object size) as the observation, and approximate the observation likelihood by counting the number of objects occupying each pixel. Instead, we assume a set of continuous detections as the observation, encode the joint state as a set, and approximate the observation function by considering all possible data-associations (with prunings). Sarkka et. al. \cite{sarkka2007rao} propose a Rao-Blackwellized particle filtering approach, under the assumption that there is a (very large) constant number of objects, while only an unknown, varying number of objects are visible. Their method encodes a possibility of data-associations in each particle, and separately updates each object within a particle using Kalman filters. Our method encodes all possible data-associations in a particle by using a set formulation, and implicitly reason about data-associations via observation likelihood in joint space. Random finite set (RFS) \cite{mahler1994random,vo2005sequential,vihola2007rao,maggio2008efficient} models MOT according to a specialized theory of finite set statistics (FISST) \cite{goodman1997mathematics}. From a mathematical point of view, central FISST concepts such as set integral and set derivative are beyond the scope of standard probability theory. Our method has some of the same advantages, but stays much simpler with only conventional probabilistic concepts. Bai et. al. \cite{bai2014intention} propose the idea of particle filtering over sets (PFS), particularly focusing on intention understanding in the domain of human-robot interaction. In this paper, we extend PFS to general MOT domains, formalize the identification problem, and present much more thorough technical details and experimental results.

\section{The Approach}
\label{sec:MOTIS-approach}
In this section, we present our main approach, namely multi-object tracking and identification over sets (MOTIS).

\subsection{A Set as a Random Variable}
Before introducing the entire approach, we first present our treatment of a set as a random variable, particularly the definition of the probability/density of observing a set. Please notice that, we may use probability and density interchangeably in this paper where there is no ambiguity.

\begin{theorem} \label{theorem:1}
Let random variable $S$ be a set of $n$ random variables $S = \{X_i\}_{i=1:n}$.
The joint probability of observing a set $S = \{x_i\}_{i=1:n}$, where $x_i$ are $n$ distinct values, is $\Pr(S) = \sum_{\sigma \in A_n} \Pr(X_1 = x_{\sigma(1)}, X_2 = x_{\sigma(2)}, \dots, X_n = x_{\sigma(n)})$, where $A_n$ is the set of all permutations of $\{i\}_{i=1:n}$, and $\Pr(X_1 = x_{\sigma(1)}, X_2 = x_{\sigma(2)}, \dots, X_n = x_{\sigma(n)})$ is the joint probability of observing $X_1 = x_{\sigma(1)} \wedge X_2 = x_{\sigma(2)} \wedge \dots \wedge X_n = x_{\sigma(n)}$.
\end{theorem}

\begin{proof}
When observing $S = \{x_i\}_{i=1:n}$, we do not know each value $x \in \{x_i\}_{i=1:n}$ comes from which random variable $X \in \{X_i\}_{i=1:n}$. The probability of observing $S$ counts all possibilities, which are all assignments from random variables to the observed values. An assignment $\psi$ is a bijection $\psi : \{X_i\}_{i=1:n} \to \{x_i\}_{i=1:n}$, corresponding to a permutation of elements of $S$.
\end{proof}

\begin{corollary} \label{corollary:1}
Let $\mathcal{O} = \{o_i\}_{i=1:n}$ be a set of $n$ distinct objects. Sampling without replacement for $k$ times, suppose the result is a set $S = \{o_{(i)}\}_{i=1:k}$. The joint probability of observing $S$ is $\Pr(S) = k! {1 \over {n (n-1) \cdots (n - k + 1)}} = {1 \over {n \choose k}}$, where ${n \choose k} = {n! \over k!(n-k)!}$ is the binomial coefficient.
\end{corollary}

\begin{corollary} \label{corollary:2}
Let $X$ be a random variable following a probability function $f_X(x)$, and $S = \{X_i\}_{i=1:n}$ be a set of random variables, whose elements are independent and identically distributed as $X$, the joint probability of observing a set $S = \{x_i\}_{i=1:n}$ is $\Pr(S) = n! \prod_{1 \le i \le n} f_X(x_i)$.
\end{corollary}

Notice that, when we say $S = \{X_i\}_{i=1:n}$ is a set with random variables with size $n$, we imply that there are no ties among values in $S$ according to the definition of a set. As an example, suppose we toss a fair coin for two times, there are 3 possible observations in terms of sets: $\{Head\}$, $\{Head, Tail\}$ and $\{Tail\}$. According to Corollary 2, $\Pr(\{Head, Tail\}) = 1/2$. However, Corollary 2 does not cover the cases with two $Heads$ and two $Tails$.

\subsection{The HMM Formalization}

\subsubsection{Motion model.}
\label{sec:hmm-modelling}
Formally, we define a joint state as a finite set of all objects, $S = \{s_i\}_{i=1:|S|}$. An object is represented as a high-dimensional vector $s = ( x, y, \dot x, \dot y)$, where $(x, y)$ and $(\dot x, \dot y)$ are the position and velocity respectively, both in world frame. We assume each object moves independently following a random-acceleration moving model: $(x, y) \gets (x, y) + (\dot x, \dot y)\tau + {1 \over 2}(\ddot x, \ddot y)\tau^2$ and $(\dot x, \dot y) \gets (\dot x, \dot y) + (\ddot x, \ddot y) \tau$, where $\tau$ is the update time interval, and $(\ddot x, \ddot y)$ is the random acceleration, computed as $(\ddot x, \ddot y) = (p \cos{\theta}, p \sin{\theta})$, where $p \sim \mathcal{N}(0, \sigma_p^2)$ is the \emph{dash} power, and $\theta \sim \mathcal{U}(0, 2\pi)$ is the \emph{dash} direction. Here, $\mathcal{N}$ and $\mathcal{U}$ denote Gaussian and uniform distributions; and, $\sigma_p^2$ is the \emph{dash} power variance. Furthermore, we model the fact that objects may occasionally move into or out from the monitoring filed as a birth-death process with birth rate $\lambda$ and death rate $|S|\mu$ per second.

\subsubsection{Observation model.}
\label{sec:observation-function}
Observations are sequentially provided by a low-level object detector as a set of detections, $O = \{o_i\}_{i=1:|O|}$. We assume a detection $o = (x, y, c)$ includes a position $(x, y)$ in world frame and a confidence value $c \in [0, 1]$. The confidence value reflects the internal classification confidence of the detector, which, for example, may come from margin distances of support vector machines (SVMs) used in the detection algorithm. If the detector can not provide confidence values, we can just use default values. Thus this is actually a general formulation.

Let's first consider the case for a single object and a single detection. Given state $s = (x, y, \dot x, \dot y)$, we denote $\Pr(o \mid s)$ as the probability of observing detection $o = (x', y', c)$, computed as $\Pr(o \mid s) = \Pr(c \mid \boldsymbol 1) \Pr(x', y' | x, y)$, where $\Pr(c \mid \boldsymbol 1)$ is the probability of having confidence $c$ given that there is truly an object, and $\Pr(x', y' \mid x, y)$ is the probability of having a detection in position $(x', y')$ given that the object is in position $(x, y)$. We use a Beta distribution to model $\Pr(c \mid \boldsymbol 1) = Beta(c \mid 2, 1)$, and a Gaussian distribution to model $\Pr(x', y' \mid x, y) = \mathcal{N}(x', y' \mid x, y, \boldsymbol\Sigma)$, where $\boldsymbol\Sigma$ is the covariance.

In the case of false detection, let $\Pr(o \mid \emptyset)$ be the probability of observing $o = (x', y', c)$, computed as $\Pr(o \mid \emptyset) = \Pr(c \mid \boldsymbol 0) f_b(x', y')$, where $\Pr(c \mid \boldsymbol 0)$ is the probability of having confidence $c$ given that there is no object, and $f_b(x', y')$ is a background distribution giving the probability that a false detection is occurring in position $(x', y')$. We use a Beta distribution to model $\Pr(c \mid \boldsymbol 0) = Beta(c \mid 1, 2)$, and a uniform distribution over the monitoring area to model $f_b$.

In general cases, we assume at a single time step, an object can result in at most one detection, and a detection can originate from at most one object. Let $F \subseteq O$ and $M \subseteq S$ be the set of false and missing detections, each possible combination of $F$ and $M$ must satisfy $|O - F| = |S - M|$. Denoted by $O \circ S = \{\langle F_i, M_i \rangle \}_{i=1:|O \circ S|}$ the set of all $F$-$M$ pairs, we have $|O \circ S| = \sum_{0 \leq i \leq \min\{|O|,|S|\}} {|O| \choose i} {|S| \choose i} = {|O| + |S| \choose |O|}$. We assume that false and missing detections are independently following Poisson processes with parameters $\nu$ and $|S|\xi$ per second. Suppose the update time interval is $\tau$, according to Corollaries \ref{corollary:1} and \ref{corollary:2}, the observation function is
\begin{multline}
\label{eq:joint-observation}
\Pr(O \mid S) = \sum_{\langle F, M \rangle \in O \circ S} \Pr(O - F \mid S - M) \cdot (\nu\tau)^{|F|} e^{-\nu\tau} \\
\prod_{o \in F} P (o \mid \emptyset) {(|S|\xi\tau)^{|M|} e^{-|S|\xi\tau} \over |M|!} {1 \over {|S| \choose |M|}},
\end{multline}
where $\Pr(O - F \mid S - M)$ gives the probability of observing the same number of detections given objects. For convenience, we define $f_F(F) = (\nu\tau)^{|F|} e^{-\nu\tau} \prod_{o \in F} P (o \mid \emptyset)$, and $f_M(M) = {(|S|\xi\tau)^{|M|} e^{-|S|\xi\tau} \over |M|!} {1 \over {|S| \choose |M|}}$ hereinafter. Let $\Psi_{S - M}^{O - F}$ be the set of all possible assignments from $S - M$ to $O - F$, assuming conditional independence between observations, we have
\begin{equation}
\label{eq:joint-observation-matched}
\Pr(O - F \mid S - M) = \sum_{\psi \in \Psi_{S-M}^{O-F}} \prod_{s \in S - M} \Pr(\psi(s) \mid s).
\end{equation}

Combining Equations \ref{eq:joint-observation} and \ref{eq:joint-observation-matched}, we have the full observation function, which has $\sum_{0 \leq i \leq \min\{|O|,|S|\}} {|O| \choose i} {|S| \choose i} i! = \Omega( ({\max\{|O|,|S|\} \over e})^{\min\{|O|,|S|\}})$ terms. It is intractable to compute the full expression in real time for even moderate state or observation sizes. Approximations are made in practice.

\subsection{Observation Function Approximation}
\label{sec:data-association}

\subsubsection{Assignment pruning.}
\label{sec:assignment-pruning}
Equation \ref{eq:joint-observation-matched} has $m!$ ($m = |S-M| = |O-F|$) terms in total, which makes it intractable in practice. Basically not all assignments need to be considered, since most of them have relatively very small probabilities compared with the best assignment, particularly for cases when $m > 2$. To this end, we convert probabilities $\Pr(o \mid s)$ to costs $c(s, o) = -\log(\Pr(o \mid s))$, and find the assignments in cost-increasing order by following Murty's algorithm \cite{murty1968algorithm} until the probability ratio of the last assignment to the first assignment is lower than a threshold. In general, optimized Murty's algorithm finds the top-$k$ best assignments of an assignment problem with size $N \times N$ in $O(kN^3)$ time complexity.

\subsubsection{False-missing pruning.}
\label{sec:fm-pruning}
The set of all possible $F$-$M$ pairs has size ${|O| + |S| \choose |O|}$, leading to a huge time complexity when computing the full observation function.  The idea is to find the possible $F$-$M$ pairs $\langle F, M \rangle$ in probability decreasing order until $f_F(F)f_M(M)$ is lower than a threshold with the help of a priority queue. The overall approximated observation function with this pruning strategy is implemented in Algorithm \ref{algo:observation-function},
where a priority queue is used to ensure that $F$-$M$ pairs are evaluated in a probability-decreasing order. The \FuncSty{Murty} function in Algorithm \ref{algo:observation-function} approximates Equation \ref{eq:joint-observation-matched} using assignment pruning strategy as described in the previous section.

\begin{algorithm}[t]\DontPrintSemicolon
    \caption{\FuncSty{ObservationFunction}}
    \label{algo:observation-function}
    \SetKwFunction{Pop}{\textbf{Pop}}
    \SetKwFunction{Murty}{\textbf{Murty}}

    \KwIn{A set of detections $O$, and a set of objects $S$}
    \KwOut{Probability of observing $O$ given $S$}

    Let $Q \gets$ a descending priority queue initially empty \\
    Let $\mathcal{F} \gets$ a list of all possible false detections $F$ \\
    Let $\mathcal{M} \gets$ a list of all possible missing detections $M$ \\
    Sort $\mathcal{F}$ according to $f_F(\cdot)$ in descending order \\
    Sort $\mathcal{M}$ according to $f_M(\cdot)$ in descending order \\
    Add $(1, 1)$ to $Q$ with priority $f_F(\mathcal{F}[1]) f_M(\mathcal{M}[1])$ \\
    Let $p \gets 0$ \\
    \Repeat{$q <$ threshold {\bf or} $Q$ is empty} {
        Let $(i, j) \gets$ \Pop{$Q$} \\
        Let $q \gets f_F(\mathcal{F}[i]) f_M(\mathcal{M}[j])$ \\
        \If {$|\mathcal{F}[i]| = |\mathcal{M}[j]|$} {
            $p \gets p + q$ \Murty{$\mathcal{F}[i], \mathcal{M}[j]$} \\
        }
        \If {$i+1 \le |\mathcal{F}|$} {
            Add $(i+1, j)$ to $Q$ with priority $f_F(\mathcal{F}[i+1]) f_M(\mathcal{M}[j])$ \\
        }
        \If {$j+1 \le |\mathcal{M}|$} {
            Add $(i, j+1)$ to $Q$ with priority $f_F(\mathcal{F}[i]) f_M(\mathcal{M}[j+1])$ \\
        }
    }
    \Return $p$ \\
\end{algorithm}

We show that Algorithm \ref{algo:observation-function} finds the $F$-$M$ pairs in a probability decreasing order. We define $f_{FM}(i, j) = f_F(\mathcal{F}[i]) f_M(\mathcal{M}[j])$ for short. In the $k$th ($1 \le k \le |\mathcal{F}||\mathcal{M}|$) iteration of the loop, let $Q_k$ be the priority queue before popping, and let $(i_k, j_k)$ be the popped element, we have $(i_k, j_k) = \argmax_{(i, j) \in Q_k} f_{FM}(i, j)$, and $Q_{k+1} \cup (i_k, j_k) = Q_k \cup \boldsymbol 1[i_k+1 \le |\mathcal{F}|] (i_k+1, j_k) \cup \boldsymbol 1[j_k+1 \le |\mathcal{M}|] (i_k, j_k+1)$. Since $f_{FM}(i_k+1, j_k) \le f_{FM}(i_k, j_k)$ and $f_{FM}(i_k, j_k+1) \le f_{FM}(i_k, j_k)$, we have $f_{FM}(i_{k+1}, j_{k+1}) \le f_{FM}(i_k, j_k)$ for $1 \le k \le |\mathcal{F}||\mathcal{M}| - 1$. So Algorithm \ref{algo:observation-function} finds the $F$-$M$ pairs in a desired probability decreasing order.

\subsection{Particle Filtering}
We use particle filter to make the inference in the formulated HMM. A particle is defined as a set of object states, $X = \{ s_i \}_{i=1:|X|}$. The posterior distribution over states $\Pr(S_t \mid O_1, O_2, \dots, O_t)$ is approximated as a set of weighted particles $\mathcal{P}_t = \{ \langle X_t^{(i)}, w_t^{(i)} \rangle\}_{i=1:N}$, such that $\sum_{i=1}^N w = 1$. In each step of updating, for particle $X_{t-1} \in \mathcal{P}_{t-1}$, a new particle is proposed from a \emph{proposal distribution}: $\hat{X}_t \sim \pi(\cdot \mid X_{t-1}, O_t)$. The motion, observation and proposal weights are computed as: $m_t = \Pr(\hat{X}_t \mid X_{t-1})$, $o_t = \Pr(O_t \mid  \hat{X}_t)$ and $p_t = \pi(\hat{X}_t \mid X_{t-1}, O_t)$ respectively. The particle weight is then updated as $w_t \gets w_{t-1} {m_t o_t \over p_t}$. Finally, a set of new particles $\mathcal{P}_t$ is generated by normalizing and resampling from $\mathcal{P}_{t-1}$.

\subsubsection{Particle refinement.}
\label{sec:particle-refinement}
In common implementations of particle filters, new particles are usually proposed directly from the motion model, in which case updating particle weights simplifies to $w_t \gets w_{t-1}  o_t$. However this simple proposal strategy does not work well in MOT domains, because for newly appearing objects, the probability that the motion based proposals will match the new detections is extremely small. To overcome this difficulty, a  refinement method is developed to make more informative proposals.

For detection $o = (x', y', c)$, the probability that it is not a false detection is $\Pr(\boldsymbol 1 \mid c) = {\Pr(c \mid \boldsymbol 1) \Pr(\boldsymbol 1) \over \Pr(c \mid \boldsymbol 1) \Pr(\boldsymbol 1) + \Pr(c \mid \boldsymbol 0) \Pr(\boldsymbol 0)} = c$ in our Beta assumptions, by assuming prior probabilities $\Pr(\boldsymbol 1) = \Pr(\boldsymbol 0) = 0.5$. Given that this detection is not a false detection, the probability that it originates from state $s = (x, y, \cdot, \cdot)$ is $\Pr(s \mid o) = \eta \Pr(o \mid s) \Pr(s) = \mathcal{N}(x', y' \mid x, y, \boldsymbol\Sigma)$ in Gaussian assumption, if the prior distribution of $s$ is assumed to be uniform. Therefore, for a new detection $o \in O$, we in principle propose object $s$ distributed as $\Pr(s \mid o)$ with probability $c$. We denote this mixture proposal distribution as $\pi_{s}(\cdot \mid o)$, with $\pi_s(\emptyset \mid o) = 1 - c$ and $\pi_s(s \mid o) = c \mathcal{N}(x', y' \mid x, y, \boldsymbol\Sigma)$, if $o = (x', y', c)$ and $s = (x, y, \cdot, \cdot)$. Notice that the velocity of $s$ is ignored in the proposal distribution. The question is, for each particle $X$, how to determine whether a detection $o \in O$ is new or it originates from an existing object $s \in X$. We find out possible new detections by seeking the most likely data-association between $X$ and $O$.

Formally, a data-association between particle $X$ and observation $O$ is defined as a 3-tuple $\varphi = \langle F, M, \psi \rangle$ where $F \subseteq X$ is the set of false detections, $M \subseteq O$ is the set of missing detections, and $\psi \in \Psi_{X-M}^{O-F}$ is an assignment from $X-M$ to $O-F$. Equation \ref{eq:joint-observation} can then be re-written as $\Pr(O \mid X) = \sum_\varphi \Pr(O, \varphi \mid X)$. Let $\varphi^* = \argmax_\varphi \Pr(O, \varphi \mid X)$ be the optimal data-association in terms of observation likelihood. Suppose $\varphi^* = \langle F^*, M^*, \psi^* \rangle$, then $F^*$ is intuitively the most likely set of new detections given $X$. The resulting proposal distribution $\pi_{r}( \cdot \mid X_{t-1}, O_t)$ is implemented as follows.
\begin{enumerate}
\item Sample $X'_t \sim \Pr(\cdot \mid X_{t-1})$ using only motion model,
\item \label{enum:sample-association} Find best data-association $\varphi^* = \langle F^*, M^*, \psi^* \rangle$ given $X'_t$,
\item Propose new objects $X' = \{ s \mid s \sim \pi_s(\cdot \mid o), o \in F^* \}$,
\item Propose a refined particle $X''_t \gets X'_t \cup X'$,
\item \label{enum:argmax} Return $\hat{X}_{t} \gets \argmax_{ X \in \{  X'_t, X''_t \}} \Pr(O_t \mid X)$.
\end{enumerate}

Note that Step \ref{enum:sample-association} can be easily approximated by running Algorithm \ref{algo:observation-function}, and Step \ref{enum:argmax} is used as an acceptance test to ensure a better proposal in terms of $O$ is returned. Temporary results of running Algorithm \ref{algo:observation-function} are cached and reused in Step \ref{enum:sample-association}, Step \ref{enum:argmax} and further update steps whenever possible. The resulting proposal strategy is very efficient at capturing new objects. It turns out that we do not need to propose any new objects in the motion model. We simply set the object birth rate to 0 in experiments.

\subsubsection{Bayesian density estimation.}
Using the refined proposal distribution, it is necessary to compute motion and proposal weights when updating particles in order to make the updated particles consistent with the underlying Bayesian filtering equation. However in particle filtering framework, we are not able to necessarily have the explicit expression of motion function to compute these weights. A Bayesian density estimation method is proposed to alleviate this difficulty.

For a set of particles $\mathcal{P} = \{X_i\}_{i=1:N}$, let $\mathcal{P}'$ be the set of particles proposed directly from motion model: $\mathcal{P}' = \{ X' \mid X' \sim \Pr( \cdot \mid X), X \in \mathcal{P} \}$, and $\mathcal{P}''$ be the set of refined particles: $\mathcal{P}'' = \{ X'' \mid X'' \sim \pi_r(\cdot \mid X'), X' \in \mathcal{P}' \}$. Following the idea of \cite{biswas2011corrective}, we estimate motion and proposal weights by seeing $\mathcal{P'}$ and $\mathcal{P}''$ as data and building density estimators over them: $\Pr(X'' \mid X) \approx \Pr(X'' \mid \mathcal{P}')$ and $\pi_r(X'' \mid X) \approx \Pr(X'' \mid \mathcal{P}'')$, where $X''$ is the refined proposal from $X$.

Assuming object states are independently distributed as an unknown distribution $f_s$, we have $\Pr(X \mid \mathcal{P}) = n! \Pr(|X| = n \mid \mathcal{P}) \prod_{s \in X} f_s(s \mid \mathcal{P})$, defined over sets according to Corollary \ref{corollary:2}. Furthermore, we assume the number of objects follows a Poisson distribution with unknown parameter $\gamma$. Suppose $\gamma$ is priorly distributed as a Gamma distribution with parameters $(\alpha_0, \beta_0 )$,  then the posterior distribution of $\gamma$ is also Gamma with updated parameters $( \alpha = \alpha_0 + \sum_{X \in \mathcal{P}} |X|, \beta = \beta_0 + N )$ by following the Bayesian method. Hence the posterior predictive of the number of objects is a negative binomial distribution with number of failures $r = \alpha$ and success rate $p = {1 \over 1 + \beta}$. That is to say, $\Pr(|X| = n \mid \mathcal{P}) = \mathcal{NB}(n ; r, p) = {n + r - 1 \choose n} p^n (1-p)^r$. A kernel density estimator (KDE) is further used to approximate $f_s(s \mid \mathcal{P})$. Denoted by $\mathcal{H}(\mathcal{P}) = \{ s \mid s \in X, X \in \mathcal{P} \}$ the set of all objects in $\mathcal{P}$, $f_s$ is approximated as $f_s(s \mid \mathcal{P}) \approx {1 \over |\mathcal{H}(\mathcal{P})|} \sum_{s' \in \mathcal{H}(\mathcal{P})} \phi(x - x') \phi(y - y')$, where $s = (x, y, \cdot, \cdot)$ and $\phi$ is the standard Gaussian function. Notice that velocities ($(\dot x, \dot y)$) are ignored in estimation and coordinates ($x$ and $y$) are assumed to be independent here.

\subsection{Object Identification}
\label{sec:object-identification}

Although a set of updated particles $\mathcal{P}_t$ encodes completely the joint posterior distribution, an object identification process is needed to identify each individual object, which provides more useful information for high-level tasks. For example, two particles may state that the joint state is either $\{s_1, s_2\}$ or $\{s_3, s_4, s_5\}$. Although this is a complete posterior distribution of joint state in terms of particles over sets, it lacks detailed information on the object level, which refers to the existence and actual state for particular objects. An object identification procedure is incorporated to provide such results as $\{object_1, object_2, object_3\}$, where $object_i = (s'_i, confidence)$, where $confidence$ is the probability that this object exists. The link between a state in a particle and an object is treated as a hidden variable. An expectation-maximization (EM) algorithm is developed to make the inference by iteratively proposing hidden variables and updating the joint state. 

 An \emph{identified object} (or \emph{identity} for short) is defined as a 3-tuple $h = (s, c, \rho)$, where $s$ is the expected state, $c \in [0, 1]$ is the confidence value, and $\rho$ is a unique ID number. Identity $h$ is estimated from an associated subset $\mathcal{H}(h) \subseteq \mathcal{H}(\mathcal{P}_t)$ (named \emph{state pool}) as: $s = {1 \over |\mathcal{H}(h)|} \sum_{ s' \in \mathcal{H}(h)} s'$, and $c = { |\mathcal{H}(h)| \over N }$. Let $L_t = \{h_i\}_{i=1:|L_t|}$ be the list of identities at cycle $t$, initially $L_0 = \emptyset$. For each $o \in O_t$, we propose a new identity $h_o$ with $\mathcal{H}(h_o)$ initially empty. Let $L_{O_t} = \{ h_o \mid o \in O_t \}$ be the set of all potential new identities, the set of all candidates to be identified at cycle $t$ is $C_t = L_{t-1} \cup L_{O_t}$. As aforementioned, each identity $h \in C_t$ is associated with a state pool $\mathcal{H}(h)$, which is equivalent to labelling each state $s$ such that $\mathcal{H}(h) = \{ s \mid l(s) = h, s \in \mathcal{H}(\mathcal{P}_t) \}$. Let $f_h$ be the state distribution of identity $h \in C_t$, and $\mathbf{P} = \{ f_h \mid h \in C_t \}$ be the set of all identity distributions, we propose an object identification process to find the best estimation of $\mathbf{P}^*$ that can optimally explain the updated particles, formally $\mathbf{P}^* = \argmax_{\mathbf{P}} \max_l \Pr(\mathcal{P}_t, l \mid \mathbf{P})$.

The EM algorithm seeks the maximum a posteriori (MAP) by iteratively applying the following two steps (in a k-means algorithm way):

{\bf E step}: $l^{(k)} = \argmax_l \Pr(\mathcal{P}_t, l \mid \mathbf{P}^{(k-1)})$,

{\bf M step}: $\mathbf{P}^{(k)} = \argmax_{\mathbf{P}} \Pr(\mathcal{P}_t, l^{(k-1)} \mid \mathbf{P})$.

The E step is equivalent to finding a labelling method for each particle $X \in \mathcal{P}_t$ such that $\prod_{s \in X} f_{l(s)} (s)$ is maximized, given $f_h \in \mathbf{P}^{(k-1)}$, which is then solved by reducing to $N$ number of best assignment sub-problems. The M step is approximated in an MLE way taking account of the current observation $O_t$. Each detection $o \in O_t$ is associated with a subset $\mathcal{H}(o) \subseteq \mathcal{H}(\mathcal{P}_t)$, constructed by selecting the most likely data-association $\langle F^*, M^*, \psi^* \rangle = \varphi^* = \argmax_\varphi \Pr(O_t, \varphi \mid X)$ given particle $X$, and updating $\mathcal{H}(\psi^*(s))$ as: $\mathcal{H}(\psi^*(s)) \gets \mathcal{H}(\psi^*(s)) \cup s$ for all $s \in X$. Notice that cached results in the particle filtering step are reused here. For each $h \in C_t$, $f^{(k)}_h$ is then computed as: $f^{(k)}_h(s) = \sum_{o \in O_t} f_h(s, o) + f_h(s, \emptyset) = \sum_{o \in O_t} \Pr(s \mid o) f_h(o) + f_h(s, \emptyset) = \sum_{o \in O_t} \mathbf{1}[s \in \mathcal{H}(o)] f_h(o) + \mathbf{1}[\forall o : s \notin \mathcal{H}(o)] f_h(\emptyset)$, where $\mathbf{1}$ is the indicator function, $f_h(o)$ is the probability that detection $o$ is generated from identified object $h$, and $f_h(\emptyset)$ is the probability that identified object $h$ does not have a detection. We then approximate $f_h(o) = \Pr(o \mid h) \Pr(h) \approx {1 \over N} |\mathcal{H}(o) \cap \mathcal{H}(h)|$, and  $f_h(\emptyset) = \Pr( \emptyset \mid h) \Pr(h) \approx { 1 \over N } \left| \mathcal{H}(h) - \bigcup_{o \in O_t} \mathcal{H}(o) \cap \mathcal{H}(h) \right| $. Therefore, we have $\mathbf{P}^{(k)} = \{ f^{(k)}_h \mid h \in C_t \}$.

The algorithm runs the M step first with $l^{(0)}$ initialized according to the converged/final labelling from the last cycle, taking account of deletion, addition and repetition of states during the particle filtering step. It then iteratively proposes a sequence of new distributions $\mathbf{P}^{(k)}$ and new labelling $l^{(k+1)}$ until convergence or a maximal number of steps is reached. The final labelling $l$ is used to construct $L_t$ as: $L_t = \{ l(s) \mid s \in \mathcal{H}(\mathcal{P}_t) \}$. Notice that $|L_t| \le |C_t|$, since candidates with finally empty state pools will not be included in $L_t$. Thus the algorithm is able to tell each individual objects with confidence values from updated particles, where each particle represents a set of potential objects.

\section{Experimental Evaluation}
\label{sec:experiments}

\begin{figure}[!t]
\begin{center}
    \begin{subfigure}{.5\linewidth}
        \includegraphics[width=\linewidth]{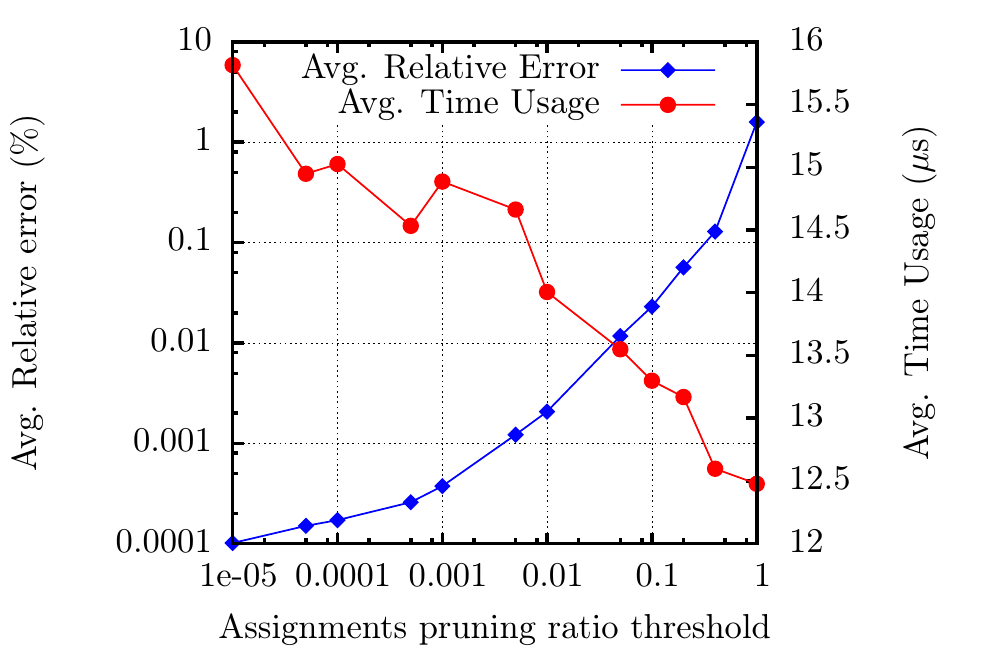}
        \caption{Assignment pruning}
        \label{fig:murty}
    \end{subfigure}%
    \begin{subfigure}{.5\linewidth}
        \includegraphics[width=\linewidth]{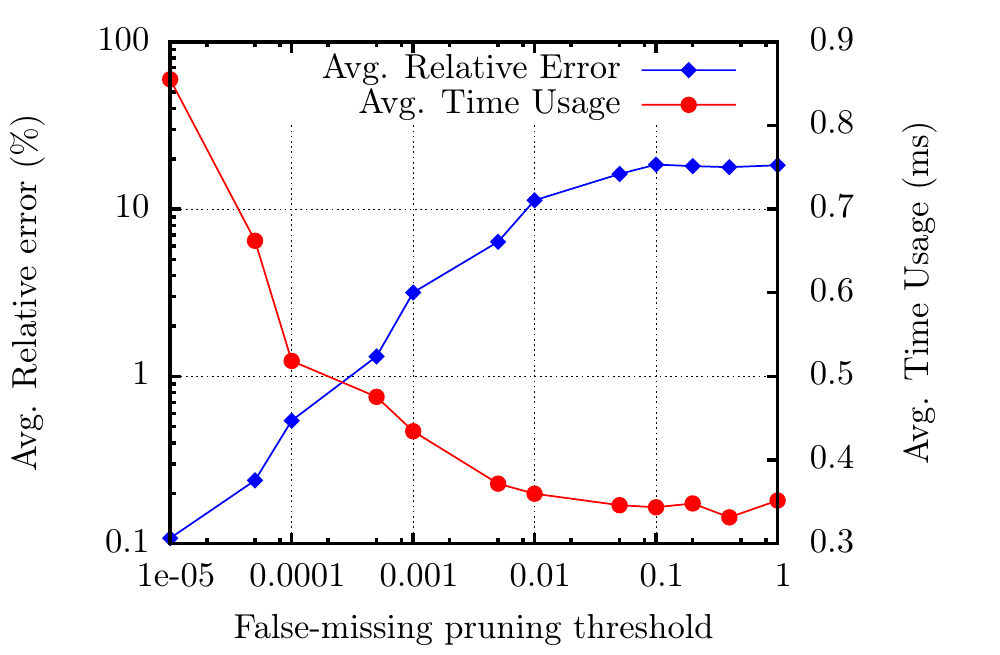}
        \caption{False-missing pruning}
        \label{fig:overall}
    \end{subfigure}
\caption{Pruning approximation error test.}
\end{center}
\end{figure}

\begin{figure*}[!t]
\centering
    \begin{subfigure}{.25\linewidth}
        \includegraphics[width=0.98\linewidth]{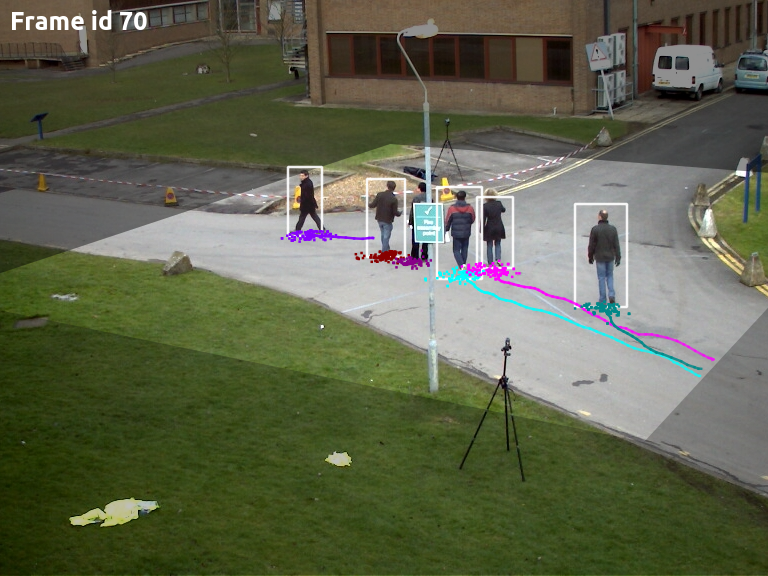}
    \end{subfigure}%
    \begin{subfigure}{.25\linewidth}
        \includegraphics[width=0.98\linewidth]{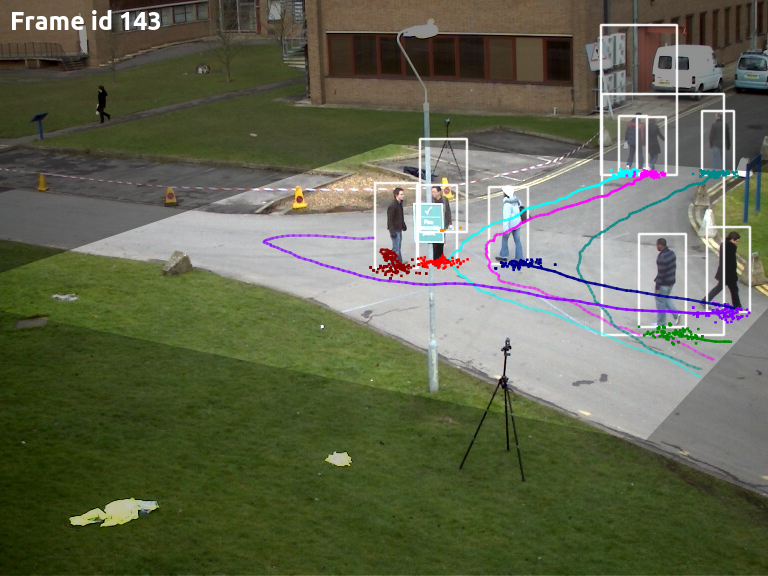}
    \end{subfigure}%
    \begin{subfigure}{.25\linewidth}
        \includegraphics[width=0.98\linewidth]{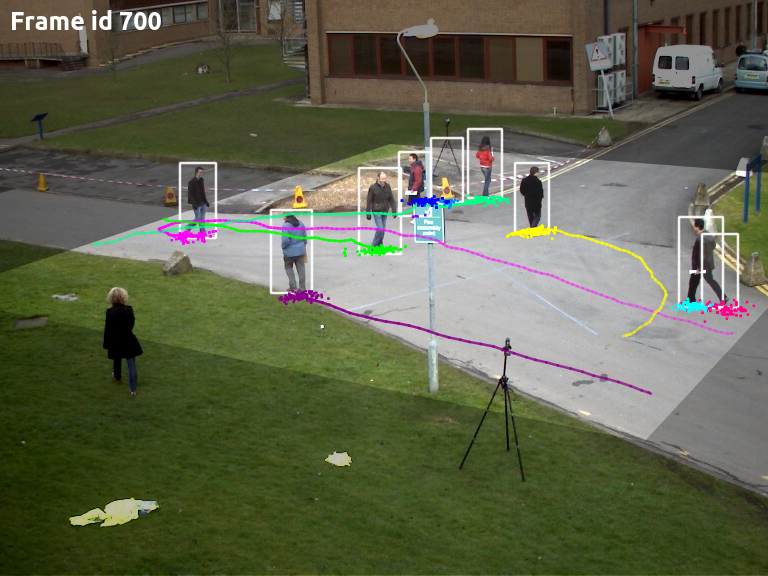}
    \end{subfigure}%
    \begin{subfigure}{.25\linewidth}
        \includegraphics[width=0.98\linewidth]{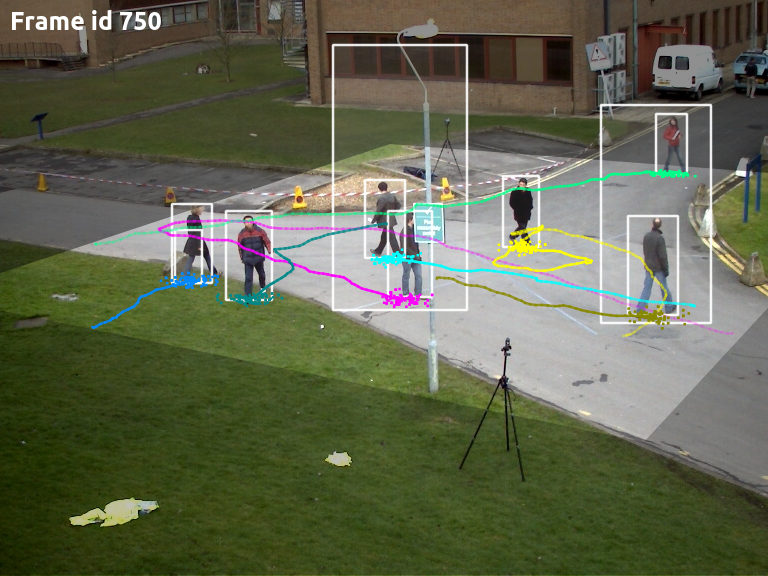}
    \end{subfigure}
\caption{Tracking results of MOTIS in PETS2009 S2L1 dataset. A video showing the whole results is available at \protect\url{http://goo.gl/4QIIey} anonymously.}
\label{fig:benchmark}
\end{figure*}

In the first experiment, we evaluate the approximation error introduced by pruning strategies. We generated a random scenario with a length of 1,000 cycles according to a birth-death process with birth rate $\lambda = 0.06$/s and death rate $\mu = 0.02$/s. When computing Equation \ref{eq:joint-observation} and Equation \ref{eq:joint-observation-matched}, we record the result with and without pruning respectively as $v$ and $v'$, then the relative error is calculated as $|{v - v' \over v}|$. Let $T'$ be the assignment pruning ratio threshold, and $T''$ be the false-missing pruning threshold. To evaluate the assignment pruning, we fixed $T''$ to be 0.001, ran the resulting approach with variable $T'$ over the generated scenario, and reported the average relative error and the average time usage. In the experiment, trivial cases in which the assignment problem was smaller than $2 \times 2$ were not counted. Figure \ref{fig:murty} depicts the results in logarithm form, from which it can be seen that, as $T'$ grows, the average relative error increases near proportionally and the average time usage decreases almost exponentially. However the average relative error stays to be rather small. It is no more than 2\% even if $T'$ is exactly chosen to be 1, in which case only top-2 assignments are calculated. To evaluate the false-missing pruning, we fixed $T'$ to be 0.1, and ran the algorithm with different $T''$ over the generated scenario. The results are shown in Figure \ref{fig:overall}. A similar trend can be observed in the figure. However $T''$ has relatively higher impact on the approximation error than ${T'}$. In the following experiments, $T'$ and $T''$ are chosen to be 0.1 and 0.001 respectively. Detailed results in this case are shown in Table \ref{tbl:pruning}.

\begin{table}[!t]
\centering
\begin{tabular}{l|l|l}
    \toprule
    Before pruning      & Equation \ref{eq:joint-observation-matched}   & Equation \ref{eq:joint-observation} \\
    \midrule
    Avg. terms & $32.66\pm 0.09$ & $1466.52 \pm 34.77$ \\
    Max. terms & $5040$ & $2.5018 \times 10^6$ \\
    \midrule
    After pruning  &     &   \\
    \midrule
    Avg. terms & $2.11 \pm 0.01$  & $29.23 \pm 0.13$ \\
    Max. terms           & $145$  & $3043$ \\
    Pruning rate        & $93.50\%$  & $97.95\%$ \\
    Relative error & $0.026\%$ & $3.30\%$ \\
   \bottomrule
\end{tabular}
\caption{Detailed results of pruning experiments with $T' = 0.1$ and $T'' = 0.001$.}
\label{tbl:pruning}
\end{table}

\begin{table*}[!t]
\centering
\tabcolsep=0.13cm
\begin{tabular}{l|l|r||l|l|r} %{@{}p{1.0in}r@{\hspace{0.5em}}r@{\hspace{0.5em}}r@{}}
    \toprule
    & Parameter & PETS2009 & & Parameter & PETS2009 \\
    \midrule
    $\lambda$ & Object birth rate (1/s) & 0.0 & $\boldsymbol\Sigma$ & Observation covariance & $0.5\boldsymbol I$ \\
    $\mu$ & Object death rate (1/s) & 0.02 & $\alpha_0$ & Initial Gamma $\alpha$ parameter & 2.0\\
    $\sigma_p$ & Dash power deviation (m$^2$/s) & 1.0 & $\beta_0$ & Initial Gamma $\beta$ parameter & 1.0\\
    $\nu$ & False detection rate (1/s) & 6.0 & $A'$ & Min. area of bounding box (m$^2$) & 0.5\\
    $\xi$ & Missing detection rate (1/s) & 2.0 & $A''$ & Max. area of bounding box (m$^2$) & 2.5\\
    $\tau$ & Update time interval (s) & 0.14 & $R$ & Min. conf. of reported identities & 0.4\\
    $T'$ & Assignment pruning threshold & 0.1 & $N$ & Number of total particles & 128\\
    $T''$ & False-missing pruning threshold & 0.001 & $H$ & Max. number of EM steps & 10\\
    \bottomrule
\end{tabular}
\caption{Parameters used in evaluation of MOTIS.}
\label{tbl:parameters}
\end{table*}

\begin{table}[!t]
\centering
\small
\begin{threeparttable}
\tabcolsep=0.14cm
\begin{tabular} {l|r|r|r|r|r} %{@{}p{1.6in}r@{\hspace{0.5em}}r@{\hspace{0.5em}}r@{\hspace{0.5em}}r@{\hspace{0.5em}}r@{}}
    \toprule
    Algorithm & MOTA & MOTP & IDS & MT & FM \\
    \midrule
    MOTIS$^1$ ({\bf proposed}) & \textbf{93.1\%} & 76.1\% & \textbf{3.6} & 18.0 & 16.0 \\
    MOTIS$^1$ $^2$ ({\bf proposed}) & 90.6\% & 74.5\% & 4.8 & 17.6 & 20.4 \\
    Milan et. al. \cite{Milan:2014:EMM} & 90.6\% & \textbf{80.2\%} & 11 & \textbf{21} & \textbf{6} \\
    Milan et al. \cite{milan2013detection} & 90.3\% & 74.3\% & 22 & 18 & 15 \\
    Segal et. al. \cite{segal2013latent} & 92\%  & 75\% & 4 & 18 & 18 \\
    Segal$^2$ et. al. \cite{segal2013latent} &  90\%  & 75\% & 6 & 17 & 21 \\
    Zamir et al.$^2$ \cite{zamir2012gmcp} &  90.3\% & 69.0\% & 8 & - & - \\
    Andriyenko et al. \cite{andriyenko2011multi} & 81.4\% & 76.1\% & 15 & 19 & 21 \\
    Breitenstein$^2$ et al. \cite{breitenstein2011online} & 56.3\% & 79.7\% & - & - & - \\
    \bottomrule
\end{tabular}
\begin{tablenotes}
\small
\item $^1$averaged over 16 runs.
\item $^2$evaluated within tracking region not cropped.
\end{tablenotes}
\caption{Quantitative results in PETS2009 S2L1 dataset.}
\label{tbl:PETS}
\end{threeparttable}
\end{table}

In the second experiment, we evaluate MOTIS in the S2L1 sequence of the challenging PETS2009 dataset \cite{ferryman2009pets2009}, to compare with existing MOT algorithms. The video is filmed with $\approx 7$ fps from a high viewpoint. It contains 795 frames, showing up to 8 ground truth humans and 13 raw detections. We mainly evaluate our algorithm in the cropped region, which covers approximately an area of $19.0 \times 15.8$m$^2$, as in \cite{andriyenko2011multi}, while also report the results in the whole area. The cropped data has at most 11 detections, and on average 5.67 detections per frame. Each detection consists a confidence value, and a bounding box with center point, height and width information in image frame. The data has a camera calibration file, so it is possible to transform raw detections into world frame. \footnote{The full dataset collected by Milan (2014) is available at http://www.milanton.de/data.html publicly.} In the experiment, we treat detections with extremely large/small bounding boxes as having confidence 0. Only identified humans with identification confidence higher than 0.4 are reported for evaluation. Table \ref{tbl:parameters} outlines the used parameters.

We evaluate the performance of MOTIS in terms of the CLEAR MOT metrics \cite{keni2008evaluating}. The distance threshold used for evaluation is 1.0m, which is widely used in literature. The Multiple Object Tracking Accuracy (MOTA) takes into account false positives, false negatives and identity switches. The Multiple Object Tracking Precision (MOTP) is simply the average distance $d$ in meters between true and estimated objects, normalized to a percentage as $100 \times (1 - d) \%$. In more detail, let $n_t$ be the number of correct matches between ground truth and tracking results found at cycle $t$ by solving a constrained assignment problem (under distance threshold 1.0m) following CLEAR MOT, $d_t^{(i)}$ be the distance between ground truth object $s_t^{(i)}$ and its corresponding tracked identity $h_t^{(i)}$, MOTP is defined as
\begin{equation}
\text{MOTP} = 1 - {\sum_t \sum_{1 \le i \le n_t} d_t^{(i)} \over \sum_t n_t }.
\end{equation} Let $g_t$ be the ground truth number of objects, $a_t$ be the number of tracked identities, and $m_t$ be the number of mismatches (i.e. identity switches) in the mapping, MOTA is defined as
\begin{equation}
\text{MOTA} = 1 - { \sum_t {\left( g_t + a_t - 2n_t + m_t \right)} \over \sum_t g_t}.
\end{equation}
It can be seen that if averaged distance between each ground truth object and its tracked identity is zero, then MOTP equals 100\%; if the number of ground truth objects equals the number of matches, the number of tracked objects equals the number of matches, and there is no ID-switch errors, then MOTA equals 100\%. MOTA and MOTP show the ability of the tracker in terms of estimating human positions and intentions under the consideration of confidence, and WMTA indicates the performance at tracking and keeping their trajectories.

Furthermore, we also report the metrics proposed in \cite{li:learning}, which counts the number of mostly tracked (MT) trajectories, track fragmentations (FM) and identity switches (IDS). An object is mostly tracked when at least 80\% of its ground truth trajectory is found. Track fragmentations count how many times a ground truth trajectory changes its status from ``tracked" to ``not tracked".
Table \ref{tbl:PETS} presents the experimental results, and Figure \ref{fig:benchmark} shows some tracking examples. In the figures, white bounding boxes are the raw detections; trajectories, and current states in particles are depicted with different colors indicating different identified humans.

In comparison, Breitenstein et. al. \cite{breitenstein2011online} track each object separately with greedy data-association via particle filtering, which can be seen as a good baseline for our method. Segal et. al. \cite{segal2013latent} model MOT as a switch linear dynamical system and take advantage of a trained pedestrian and outlier detector in the object domain. Zamir et. al. \cite{zamir2012gmcp} utilize generalized minimum clique graphs to solve the data-association problem by incorporating both motion and appearance information. Andriyenko et. al. \cite{andriyenko2011multi}, Milan et. al. \cite{Milan:2014:EMM} and Milan et. al. \cite{milan2013detection} formulate MOT as an offline optimization problem over splines given an energy function, and initial tracks obtained from per-object extended Kalman filters given greedy data-associations. It can be seen from the results that, our algorithm outperforms the state-of-the-art, being able to run online in a Bayesian recursive way without using any data dependent information (such as object appearance).

\section{Conclusion}
\label{sec:conclusion}
We present a novel multi-object tracking and identification over sets (MOTIS) approach to the multi-object tracking problem. From a multi-object tracking point of view, our approach avoids directly performing observation-to-object association by using a set formulation, and inferring the posterior distribution, as well as the data-association, in joint space. The overall method outperforms the state-of-the-art in the challenging PETS2009 dataset
in terms of overall tracking accuracy and ID-switch errors. In future work, we plan to apply MOTIS in more realistic domains and conduct more systematic theoretical and experimental evaluations.

\appendix

\section*{Acknowledgements}
\bibliographystyle{abbrv}
%\bibliography{uai16.bib}

\end{document}